\documentclass[11pt]{article}
\usepackage[utf8]{inputenc}
\usepackage{amsmath, amsthm, amssymb}
\usepackage{graphicx}
\usepackage{subcaption}
\usepackage{booktabs}
\usepackage{siunitx}
\usepackage[margin=1.1in]{geometry}
\usepackage{enumitem}
\usepackage{tcolorbox}
\usepackage{siunitx}   % \SI, \peta\byte, \tera\byte, \minute, …

\usepackage{hyperref}
\usepackage{newunicodechar}
\newunicodechar{ }{\,}
\DeclareSIUnit\year{yr}
\title{Can adversarial attacks by large language models be attributed?}
\author{%
  \begin{tabular}{@{}c c c@{}}
    Manuel Cebrian\textsuperscript{1} &
    Andres Abeliuk\textsuperscript{2} &
    Jan Arne Telle\textsuperscript{3}
  \end{tabular}\\[4pt]
  {\small 
    \textsuperscript{1}Center for Automation and Robotics, Spanish National Research Council, Spain}\\
    {\small 
    \textsuperscript{2}Department of Computer Science, University of Chile, Chile}\\
     {\small 
    \textsuperscript{3}Department of Informatics, University of Bergen, Norway}
}% ← closes \author

\date{}
\newtheorem{theorem}{Theorem}

\newtheorem{proposition}[theorem]{Proposition}
\newtheorem{definition}[theorem]{Definition}
\newtheorem{corollary}[theorem]{Corollary}
\newtheorem{observation}{Observation}

\begin{document}
\maketitle

\begin{abstract}
Attributing outputs from Large Language Models (LLMs) in adversarial settings---such as cyberattacks and disinformation campaigns---presents significant challenges that are likely to grow in importance. We approach this attribution problem from both a theoretical and empirical perspective, drawing on formal language theory (identification in the limit) and data-driven analysis of the expanding LLM ecosystem. By modeling an LLM's set of possible outputs as a formal language, we analyze whether finite samples of text can uniquely pinpoint the originating model. Our results show that under mild assumptions of overlapping capabilities among models, certain classes of LLMs are fundamentally \emph{non-identifiable} from their outputs alone. We delineate four regimes of theoretical identifiability: (1) an infinite class of deterministic (discrete) LLM languages is not identifiable (Gold's classical result from 1967); (2) an infinite class of probabilistic LLMs is also not identifiable (by extension of the deterministic case); (3) a finite class of deterministic LLMs is identifiable (consistent with Angluin’s tell-tale criterion); and (4) even a finite class of probabilistic LLMs can be non-identifiable (we provide a new counterexample establishing this negative result). Complementing these theoretical insights, we quantify the explosion in the number of plausible model origins (\emph{hypothesis space}) for a given output in recent years. Even under conservative assumptions (each open-source model fine-tuned on at most one new dataset), the count of distinct candidate models doubles approximately every 0.5~years, and allowing multi-dataset fine-tuning combinations yields doubling times as short as 0.28~years. This combinatorial growth, alongside the extraordinary computational cost of brute-force likelihood attribution across all models and potential users renders exhaustive attribution infeasible in practice. Our findings highlight an urgent need for new strategies and proactive governance to mitigate risks posed by un-attributable, adversarial use of LLMs as their influence continues to expand.
\end{abstract}

\begin{tcolorbox}[colback=gray!10!white, colframe=black, boxrule=0.5mm, arc=3mm, title=\textbf{Significance}]
When AI-generated attacks---from disinformation to cyberattacks---occur, can we reliably trace them back to their originating language model? This paper establishes theoretical limits, showing that in realistic settings, attributing outputs to specific large-language models is provably impossible, even with unlimited data. Empirically, we quantify the explosive growth in the number of plausible model origins, demonstrating how quickly attribution becomes infeasible in practice. These combined results have stark implications for cybersecurity, misinformation mitigation, and AI governance.
\end{tcolorbox}

\section{Introduction}
The challenge of attributing outputs from LLMs in the context of adversarial attacks or disinformation campaigns is emerging as a concern for both cybersecurity and information integrity~\cite{axelrod2014timing,edwards2017strategic,urbina2022dual,perlroth2021they,xu2024autoattacker}. In such settings, \emph{attribution} refers to identifying the specific model responsible for generating harmful or misleading content. This step is essential not only for conducting investigations and determining whether the implicated model should be restricted or decommissioned, but also for mitigating future risks and ensuring accountability in the deployment of LLM-based agents~\cite{rahwan2019machine,shavit2023practices,anwar2024foundational}. Unfortunately, reliably linking a piece of content to a particular LLM has proven extremely challenging in practice. 

The demand for robust attribution is underscored by new AI-governance initiatives. The EU AI Act and U.S. Executive Order 14110 both mandate model-level transparency and risk-mitigation tools, such as watermarking and incident reporting, that implicitly assume one can identify which model produced a given output \cite{EUAIAct2024,EO14110}. Our work asks whether that assumption is even feasible in principle. Ultimately, accountability lies with human actors, but pinpointing the source model is a crucial intermediate step—it enables enforcement of regulations and can lead investigators to the responsible parties through the chain-of-custody of AI tools.

Interestingly, the attribution task can be framed in terms of formal language theory, specifically the problem of \emph{language identification in the limit}. This theoretical framework, introduced by Gold~\cite{gold1967language} and extended by Angluin~\cite{angluin1980inductive}, has been widely studied in theoretical computer science and cognitive science~\cite{johnson2004gold}. In our context, we can represent the set of all possible outputs of a given LLM as a formal language (a set of strings over some finite alphabet). Attribution then asks whether a finite sample of observed outputs can uniquely determine which language (and hence which LLM) produced them.

Framing LLM outputs as formal languages provides a structured way to explore the feasibility of attribution. Given an observed set of outputs $S$ (e.g., a collection of generated texts), we are essentially asking if there exists a unique LLM $M$ in some model class $\mathcal{M}$ whose language $L(M)$ includes $S$. If two different models $M_i$ and $M_j$ can both generate all strings in $S$ (i.e.~$S \subseteq L(M_i) \cap L(M_j)$), then $S$ alone cannot distinguish between $M_i$ and $M_j$. In practice, fine-tuned models often exhibit substantial overlap in their output spaces, especially if they share training data or base architectures. This overlap means that, under mild assumptions, multiple models may produce the same set of outputs, foreshadowing fundamental limits on attribution certainty. 

 Our theoretical results are derived under worst-case assumptions using Gold’s identification-in-the-limit framework. While this infinite-data scenario may seem idealized, it is in fact the most favorable setting for defenders: if attribution fails even with unlimited, adversarially selected samples, it will surely fail in practical, data-scarce conditions. This framing allows us to prove robust impossibility results that remain valid even when scaled down to real-world constraints. Analyzing worst-case scenarios is a common approach in cybersecurity and cryptography – it sets fundamental limits that remain valid even when conditions are less adversarial. Our use of an adversarial, infinite-sequence model is in this spirit: it defines the boundary of what is theoretically possible for attribution under the most challenging conditions.

In the remainder of this paper, we investigate the theoretical limits of LLM identification under four regimes and then examine empirical trends that exacerbate the attribution problem. Below we summarize these regimes and our main findings for each. 

\begin{itemize}[leftmargin=2em]
    \item \textbf{Infinite discrete model class: Negative.} If the space of possible LLMs (or languages) is infinite and models produce outputs in a deterministic (discrete) manner, then identification in the limit is not possible. This was proved in Gold’s classic result and formalized by Angluin’s criteria: intuitively, any infinite class of languages that is sufficiently complex (e.g., containing an infinite language with arbitrarily many finite variants) is \emph{not} identifiable from positive data.
    \item \textbf{Infinite probabilistic model class: Negative.} Allowing models to be probabilistic (assigning probabilities to strings) does not improve identifiability when the class is infinite. In fact, the deterministic case is a special case of the probabilistic case (a formal language can be viewed as a probabilistic language with 0/1 probabilities), so an infinite collection of probabilistic LLMs remains non-identifiable in the limit.
    \item \textbf{Finite discrete model class: Positive.} If the number of candidate LLMs is finite and their outputs are deterministic languages, then identification in the limit becomes possible. In this scenario, because there are only finitely many possible languages, one can eventually find a finite subset of outputs (a \emph{tell-tale set} in Angluin’s terminology) for each language. Thus, with enough data, a learning algorithm can converge to the correct model.
   \item \textbf{Finite probabilistic model class: Negative.}  Somewhat counter-intuitively, even a finite set of probabilistic LLMs can defy identification.  We give the first explicit counterexample—thereby resolving a question left open since Gold’s 1967 work—showing that if two probabilistic languages share identical support but differ in their probability distributions, no amount of data (under the standard identification-in-the-limit setting) can reliably distinguish them.
\end{itemize}

Following our theoretical analysis, we present a data-driven study of the current LLM landscape. The number of publicly known LLMs and fine-tuned variants has exploded in recent years~\cite{bommasani2023ecosystem-graphs}, which greatly enlarges the hypothesis space for attribution. We introduce a combinatorial lower bound $N(t)$ on the number of distinguishable model origins at time~$t$ and find that $N(t)$ has been growing exponentially, with a troublingly short doubling time (well under one year in recent data). We break down this growth by model modality and by developer region, revealing that multimodal models and contributions from Asia are among the fastest-growing segments. This rapid proliferation means that any brute-force or exhaustive attribution strategy (e.g. comparing an output against every possible model) will become increasingly infeasible.

We also examine the computational hurdles to attribution. Even under optimistic assumptions, performing likelihood-based attribution across all models for a single piece of content could require an astronomical number of operations, pushing the limits of modern supercomputers. We illustrate this with a scenario using the current cumulative parameter count of known models and show that attributing a moderately long text (100k tokens) against all models would take on the order of minutes on the world’s fastest supercomputer. Scaling such analysis to nation-wide LLM usage (on the order of $10^{15}$ tokens/year for the USA) would require infrastructure and time on the order of many days of high-performance computing for just a single attribution query, as summarized in our estimates. Moreover, real-world factors like network propagation of content can further obscure attribution, as malicious actors can route outputs through layers of social networks to mask their origin~\cite{newman2006structure,christakis2009connected}. Paradoxically, recent theoretical work has shown that \emph{language generation in the limit} is achievable without identification~\cite{kleinberg2024language}, meaning an agent can eventually mimic a target language’s outputs without actually knowing which language it is---pointing to a potential arms race where attackers and defenders can reproduce content indefinitely without exposing the true source~\cite{kleinberg2024language}.

In summary, our contributions are: (i) a rigorous theoretical exposition of why LLM attribution is impossible in three of four fundamental regimes (with a positive result only in the trivial finite-deterministic case), including a new theorem for the probabilistic finite case; (ii) a quantitative analysis of the LLM model landscape growth, demonstrating an unsustainable explosion in the attribution search space; and (iii) a discussion of practical challenges and implications, highlighting the need for new methodologies (e.g. model fingerprinting, heuristic narrowing of candidates, regulation) to address the forensic blind spot created by increasingly ubiquitous LLMs. 

Although practitioners have long suspected attribution is hard, there were no formal guarantees quantifying how hard—or under which conditions it is impossible. Our results close this gap by establishing the first rigorous limits on LLM attribution.

\section{Infinite Classes of Discrete LLMs: Impossibility of Identification}

We first consider the classical scenario of Gold~\cite{gold1967language}: an infinitely large stream of outputs from an LLM is observed (so every string the model can produce will eventually appear in the sample), and the class of potential models is infinite. In this section, we assume each model $M$ produces a \emph{discrete language} $L(M)$---a set of strings (the model's outputs) with no probabilistic information attached. Identification in this setting means that a learning algorithm, given enough data, will eventually infer the correct language $L(M)$ (and thus the correct model).

Gold formalized \emph{identification in the limit} from positive examples as follows:

\begin{definition}[Gold’s Identification in the Limit~\cite{gold1967language}]
A class of languages $\mathcal{L}$ is said to be \emph{identifiable in the limit} if there exists a learning algorithm $\mathcal{A}$ such that, for any target language $L^* \in \mathcal{L}$, given an infinite sequence of examples $\langle s_1, s_2, \dots \rangle$ with each $s_i \in L^*$ and each string in $L^*$ appearing at least once in the sequence, the algorithm $\mathcal{A}$ produces a sequence of hypotheses $\langle H_1, H_2, \dots \rangle$ (where each $H_n$ is a language in $\mathcal{L}$) that satisfies:
\begin{enumerate}[label=(\arabic*)]
    \item For all but finitely many $n$, $H_n = L^*$.
    \item Each hypothesis $H_n$ is consistent with the data observed up to time $n$, i.e. $\{s_1, s_2, \dots, s_n\} \subseteq H_n$.
\end{enumerate}
\label{def:gold_original}
\end{definition}

In simpler terms, identifiability in the limit means that as the algorithm sees more and more outputs (eventually seeing every output that the target model can produce), it converges to correctly guessing the target language and does not later change its mind. Gold showed that certain classes of languages are \emph{not} identifiable in the limit from positive data. In particular, any class of languages that is sufficiently rich---for example, containing all finite languages and at least one infinite language---cannot be learned with this criterion. 

Formally framing the attribution task in Gold’s identification-in-the-limit paradigm (as we do here) provides a unified lens to analyze attribution, connecting earlier empirical approaches (e.g. watermarking, stylometry) to a common theoretical foundation.

Angluin later provided a characterization of identifiable classes with her concept of \emph{tell-tale sets}~\cite{angluin1980inductive}. We recall Angluin’s theorem here, as it will be useful for framing our results:

\begin{theorem}[Angluin’s Theorem~\cite{angluin1980inductive}]
An indexed family of recursive languages $\{L_i\}_{i \in \mathbb{N}}$ is identifiable in the limit from positive data if and only if there exists a recursively enumerable set of finite subsets $\{T_i\}_{i \in \mathbb{N}}$ (with $T_i \subseteq L_i$ for each $i$) such that for all $i \neq j$, $T_i \not\subseteq L_j$. In other words:
\begin{enumerate}[label=(\roman*)]
    \item For each $i$, $T_i$ is a finite subset of $L_i$.
    \item For each pair $i \neq j$, if $T_i \subseteq L_j$ then $L_j$ is not a proper subset of $L_i$.
\end{enumerate}
The sets $T_i$ are sometimes called \emph{tell-tale sets} for the language $L_i$.
\label{thm:angluin}
\end{theorem}

Intuitively, Angluin’s theorem says that a class of languages is learnable from positive data if and only if each language $L_i$ in the class has some finite “evidence”, the tell-tale subset $T_i$. The point  is that once the strings of $T_i$ have appeared among the
sample strings, we need not fear ``overgeneralization" in guessing $L_i$. This is
because the true answer, even if it is not $L_i$, cannot be a proper subset of $L_i$, and so if the true answer is not $L_i$ we will eventually see a conflict between the data and $L_i$, which will
force us to change our guess.
%If such evidence exists (and can be enumerated by the algorithm), then the learner can eventually identify the correct language.
On the other hand, if no finite tell-tale sets exist then the class cannot be learned.

Using this criterion, we can formalize Gold’s negative result as a corollary. Specifically, if a class of languages contains an infinite language that has infinitely many finite subsets extendable to different languages in the class, then no finite tell-tale set can exist for that infinite language. This leads to non-identifiability:

\begin{corollary}[Non-Identifiability Due to Infinite Languages]
\label{cor:non_identifiability}
Let $\mathcal{L}$ be a collection of languages such that:
\begin{enumerate}[label=(\roman*)]
    \item $\mathcal{L}$ contains at least one infinite language $L_{\infty}$.
    \item For every finite subset $S \subset L_{\infty}$, there exists some language $L' \in \mathcal{L}$ with $S \subseteq L' \subset L_{\infty}$ (a proper subset).
\end{enumerate}
Then $\mathcal{L}$ is not identifiable in the limit from positive data.
\end{corollary}

\begin{proof}
Suppose, for sake of contradiction, that $\mathcal{L}$ is identifiable in the limit. By Theorem~\ref{thm:angluin}, there must exist a finite tell-tale set $T_{L_{\infty}} \subset L_{\infty}$ that distinguishes $L_{\infty}$ from all its proper subsets in $\mathcal{L}$. However, condition~(ii) guarantees that for \emph{every} finite subset $T_{L_{\infty}}$ of $L_{\infty}$, we can find another language $L' \in \mathcal{L}$ such that $T_{L_{\infty}} \subseteq L' \subset L_{\infty}$. This means no finite subset of $L_{\infty}$ can serve as a unique identifier, contradicting the requirement for identification. Therefore, $\mathcal{L}$ is not identifiable in the limit.
\end{proof}

Corollary~\ref{cor:non_identifiability} is essentially a formal restatement of Gold's observation: if an infinite language can be approximated arbitrarily well by other languages in the class (by matching it on every finite sample but eventually diverging), a learner can never be sure which language is the true target. 

We can construct a very simple example of such a class $\mathcal{L}$ to illustrate the concept. Consider an alphabet $\Sigma = \{x\}$ (a single symbol). For each $k \in \mathbb{N}$, define $L_k = \{x^n : 0 < n \le k\}$ as the set of all strings of length at most $k$ (over $\Sigma$). Let $L_{\infty} = \Sigma^*$ be the set of all finite strings over $\Sigma$. Now take $\mathcal{L} = \{L_k : k \in \mathbb{N}\} \cup \{L_{\infty}\}$. In this family, $L_{\infty}$ is infinite, and for any finite sample of strings $S \subset L_{\infty}$, if $m$ is the length of the longest string in $S$, then $S \subseteq L_m \subset L_{\infty}$ with $L_m \in \mathcal{L}$. Thus $L_{\infty}$ has no finite tell-tale set (any finite set of examples from $L_{\infty}$ could have come from some $L_k$ with $k$ large enough). By Corollary~\ref{cor:non_identifiability}, $\mathcal{L}$ is not identifiable. Indeed, a learner seeing strings of increasing length can never be sure whether eventually some maximal length will appear (indicating a finite language $L_k$) or the strings will continue indefinitely (indicating $L_{\infty}$).

Translating back to LLMs, we may fine-tune a base model on \emph{any}
finite set $S$ of strings and obtain a model $M_S$ whose language is
exactly $S$.\footnote{For simplicity we state the argument with
$L(M_S)=S$; relaxing to the weaker—but still sufficient—condition
$L(M_S)\supseteq S$ leaves the combinatorial lower bound unchanged.}
This corresponds to an open-ended model class in which one model
(e.g.\ a large base checkpoint) can generate an unbounded set of
outputs, yet for every \emph{finite} collection of outputs an attacker
might observe there exists another model (a suitably fine-tuned or
restricted version) that produces precisely those outputs and no
others.  Under such conditions no attribution algorithm can reliably
distinguish the unbounded model from the myriad fine-tuned variants.
In an ecosystem where models are continually specialized, this
“infinite ladder” of ever-narrower languages is not merely theoretical
but an expected consequence of routine fine-tuning.

Thus, we conclude that when considering an infinite class of possible LLMs (with languages that can nest in this way), identification in the limit is provably impossible. This negative result sets a theoretical upper bound on what we can hope to achieve with attribution: even under idealized conditions (infinite data, no noise, etc.), there are fundamental ambiguities that cannot be resolved.

\section{Infinite Classes of Probabilistic LLMs}
In practice, LLMs are probabilistic by nature---they produce distributions over outputs. One might wonder if incorporating probability information could help distinguish models where pure language membership cannot. In an identification-in-the-limit setting, however, we typically assume the learner has access only to which outputs appear, not the true underlying probabilities. (We will later discuss alternative learning criteria where samples are drawn according to the model’s probabilities rather than adversarially.) Under the standard definition, if the class of models is infinite, introducing probabilities does not overcome the fundamental obstacle identified above.

To formalize this, we first define what we mean by a probabilistic language and identification in that context:

\begin{definition}
A \emph{probabilistic formal language} $L_p$ over alphabet $\Sigma$ is a probability distribution $P$ on $\Sigma^*$ (the set of all finite strings over $\Sigma$) such that $P(s) > 0$ if and only if $s$ is in the support of $L_p$ (denoted $\mathop{\mathrm{supp}}(L_p)$). We say a string $s$ is \emph{accepted} by $L_p$ if $P(s) > 0$. The support $\mathop{\mathrm{supp}}(L_p)$ is then a (deterministic) formal language consisting of all strings the probabilistic language can produce with non-zero probability. 
\end{definition}

\begin{observation}
Any ordinary (deterministic) formal language $L$ can be viewed as a probabilistic formal language that assigns probability $1$ to each string in $L$ (normalized uniformly or in any arbitrary way) and probability $0$ to strings outside $L$. In other words, deterministic languages are a special case of probabilistic languages (with probabilities restricted to 0 or 1).
\end{observation}

Given this observation, we see that the class of probabilistic languages strictly generalizes the class of deterministic languages: every discrete language corresponds to many possible probabilistic languages that have that language as their support. Therefore, if an infinite class of deterministic languages is not identifiable, then an infinite class of probabilistic languages that includes those (as 0-1 special cases) will also not be identifiable. Any learning algorithm for the probabilistic case would in particular solve the deterministic case, which we know is impossible in the scenario above.

We can extend the definition of identification in the limit to probabilistic languages. One natural way is to require the learner to output not just a single hypothesis language at each step but a probability distribution over candidate languages, reflecting uncertainty, and to converge in probability to the correct language. A rigorous definition (adapted from Definition~\ref{def:gold_original}) is as follows:

\begin{definition}[Identification in the Limit for Probabilistic Languages]
\label{def:gold_prob}
A class of probabilistic languages $\mathcal{L}_p$ is \emph{identifiable in the limit (from positive data)} if there exists a learning algorithm $\mathcal{A}$ such that for any target probabilistic language $L^*_p \in \mathcal{L}_p$ with support $L^* = \mathop{\mathrm{supp}}(L^*_p)$, given an infinite sequence of example strings $\langle s_1, s_2, \dots \rangle$ where each $s_i \in L^*$ and each string in $L^*$ appears at least once, the algorithm produces a sequence of probability distributions over $\mathcal{L}_p$, $\langle P_1, P_2, \dots \rangle$, with the property that:
\begin{enumerate}[label=(\arabic*)]
    \item For all but finitely many $n$, the most likely hypothesis under $P_n$ is the true language $L^*_p$. (Formally, $\arg\max_{L_p \in \mathcal{L}_p} P_n(L_p) = L^*_p$ for all sufficiently large $n$.)
    \item At all times $n$, the support of $P_n$ (the set of hypotheses given non-zero probability) consists only of languages that are consistent with the data observed so far. That is, if $P_n(L_p) > 0$ then $\{s_1,\dots,s_n\} \subseteq \mathop{\mathrm{supp}}(L_p)$.
\end{enumerate}
\end{definition}

This definition ensures that eventually the learner assigns highest confidence to the correct probabilistic language, while always ruling out any languages that have already been contradicted by the observations. If such an identification is possible, we would say the class is identifiable in this probabilistic sense.

However, using the observation above, if we have an infinite class of probabilistic languages $\mathcal{L}_p$ that includes an infinite deterministic sub-class of the kind described in Section~2, then $\mathcal{L}_p$ cannot be identifiable. In particular, consider any scenario with an infinite sequence of possible models such that one model’s support language contains another’s, which contains another’s, and so on (like $L_{\infty} \supset L_{k} \supset L_{k-1} \supset \cdots$). If the learner sees all strings from the smallest language in that chain, it has also seen strings from all larger ones; without probability information in the sample selection, it cannot tell whether it’s receiving data from the smallest language or a larger one, since all observed strings are consistent with either being the case. The probabilistic aspects of the target model do not manifest in the \emph{set} of observed strings—only in how frequently they might appear, which in the adversarial presentation model is not specified. The identification in the limit framework (as defined) is adversarial in that the sequence of examples can be arranged in any order as long as every string eventually appears; in particular, it does not assume the examples are drawn from the model's own distribution.

In conclusion, any impossibility that holds for infinite deterministic classes carries over to infinite probabilistic classes. The safe assumption is:
\begin{quote}
\emph{If the class of candidate LLMs is unbounded (infinitely many possible models/languages), then no general attribution algorithm can identify the source model with certainty, even if models are probabilistic.}
\end{quote}
This result reinforces the pessimistic outlook for attribution in a scenario with open-ended model classes. It suggests that to have any hope for theoretical identifiability, one must drastically restrict the hypothesis space---for instance, by assuming only a finite (and manageable) set of candidate models.

\section{Finite Classes of Discrete LLMs: Identifiability}
We now turn to the case where the number of candidate models (and hence candidate languages) is finite. This scenario is much more favorable. If $\mathcal{L} = \{L_1, L_2, \dots, L_N\}$ is a finite set of languages, the tell-tale sets required for identification by Angluin's condition can easily be constructed. 
The target $L^*=L_i$ is one of the $N$ possibilities, and constructing the tell-tale set $T_i$ for $L_i$ is now easy since there are only a finite number of languages properly contained in $L_i$, and for each such $L_j$ there must be a string in $L_i-L_j$ that we can add to $T_i$. Thus $T_i$ is a finite subset of $L_i$ satisfying the tell-tale set condition of Angluin.

%any time we observe a string that one of the other $L_j$ cannot produce, we can rule that $L_j$ out. With only finitely many models, we only need to rule out $N-1$ of them; and since each pair of distinct languages $L_i \neq L_j$ differs by at least one string (some string belongs to one language but not the other), the target language $L^*$ has a finite “witness” against each other language. Taking the union of those witness strings for all $j \neq *$ yields a finite set $T^* \subset L^*$ that uniquely identifies $L^*$ (this $T^*$ can serve as a tell-tale set in Angluin’s sense). Thus, Angluin’s condition is satisfied in a trivial way if the class is finite.

More formally, we can state:

\begin{proposition}
Any finite collection of languages $\mathcal{L} = \{L_1, L_2, \dots, L_N\}$ is identifiable in the limit from positive data (assuming each $L_i$ is recursive, so that an algorithm can test membership).
\end{proposition}

\begin{proof}[Proof (Sketch)]
As the collection is finite the tell-tale sets required for identification by Angluin's condition can be constructed as follows. Consider a language $L_i \in \mathcal{L}$ and construct the tell-tale set $T_i$ for $L_i$ as follows. 
Consider the finite number of languages properly contained in $L_i$, and for each such $L_j$ take a string in $L_i-L_j$ and add it to $T_i$. When we are done $T_i$ is a finite subset of $L_i$ satisfying the tell-tale set condition of Angluin's Theorem \ref{thm:angluin}.
%Because the collection is finite, for each pair of distinct languages $L_i \neq L_j$, there exists at least one distinguishing string $w_{ij}$ that lies in one of the two languages but not the other. Without loss of generality, suppose $w_{ij} \in L_i$ and $w_{ij} \notin L_j$. Now consider the target language $L^* = L_k$. Collect all distinguishing strings where $L^*$ differs from another language: $T_k = \{\,w_{k\ell} : L_\ell \neq L_k\,\}$. This $T_k$ is a finite set (at most $N-1$ strings, one for each other language), and by construction $T_k \subseteq L_k$. Moreover, for any other language $L_j$, $T_k$ contains a string $w_{kj}$ that is in $L_k$ but not in $L_j$, hence $T_k \not\subseteq L_j$. Thus, $T_k$ meets Angluin’s criteria as a tell-tale set for $L_k$. An algorithm can identify $L_k$ by waiting until all strings in $T_k$ have appeared in the sample (which will happen after finitely many observations, since the sample is exhaustive by assumption). Once it has seen all those strings, it can safely output $L_k$ as the hypothesis and will not change thereafter. This algorithm will succeed for whichever $L_k$ is the target, so the class is identifiable.
\end{proof}

In summary, when the universe of possible LLMs is \emph{finite}, the attribution problem is solvable in theory. After seeing enough outputs, the attacker (or attribution algorithm) can, in effect, find a signature of the source model’s language. This result assumes we know the exact set of possible models in advance and that they have distinct output capabilities. In practical terms, this might correspond to a scenario where a small number of specific models are suspect (for example, a fixed set of known bots or generators that might have produced a given text). In such cases, especially if the models are sufficiently different, targeted attribution can succeed by cross-checking the observed outputs against each model’s known outputs.

It is worth noting, however, that Angluin’s theorem is non-constructive in the sense that it guarantees the existence of tell-tale sets but doesn’t necessarily provide an easy way to find them. In practice, even if $N$ is modest, discovering the distinguishing outputs between each pair of models might be difficult without additional information or oracles (like black-box query access to models). Nevertheless, from a purely information-theoretic viewpoint, the finite case poses no fundamental barrier to identification: given unlimited data, one can eventually tease apart any finite number of distinct languages.

Unfortunately, this optimistic scenario breaks down even with very slight relaxations of assumptions. We will see that if we add probabilistic variation to the outputs, even just two models can become impossible to distinguish in the limit.

\section{Finite Classes of Probabilistic LLMs: A Counterexample}
We now address the most subtle regime: a finite collection of probabilistic LLMs. One might hope that with only a finite number of models to consider, attribution remains feasible (as in the deterministic case). However, probability distributions can introduce ambiguity that does not occur with deterministic languages. In particular, different models can share the same support (i.e. they can all produce the same set of strings, but with different probabilities). If two models $M_1, M_2$ have $\mathop{\mathrm{supp}}(L_p(M_1)) = \mathop{\mathrm{supp}}(L_p(M_2))$ (they can produce exactly the same strings, just with different likelihoods), then any sequence of outputs that is not annotated with likelihood information cannot distinguish them---because whatever string appears, it could have come from either model. The only way to tell them apart would be to notice differences in the relative frequency or probability of outputs, but in the identification-in-the-limit framework the presentation of examples is controlled adversarially (we only assume every possible string eventually appears, not that we see them according to the model’s true distribution).

We present a simple but striking example demonstrating that even two probabilistic languages can foil any identification algorithm under Gold’s paradigm:

\begin{theorem}
Even for a finite number of probabilistic formal languages (as few as two), identification in the limit does not always hold.
\label{thm:probabilistic_negative}
\end{theorem}

\begin{proof}
Consider an alphabet consisting of a single symbol $\Sigma = \{x\}$. We define two probabilistic languages $L_{p,1}$ and $L_{p,2}$ over $\Sigma$ as follows. Both $L_{p,1}$ and $L_{p,2}$ accept \emph{every non-empty string} over $\Sigma$ (so their support is the same language $L = \{x, x^2, x^3, \dots\}$). However, they differ in the probability distribution assigned to these strings. Let $P_1$ and $P_2$ denote the probability mass functions for $L_{p,1}$ and $L_{p,2}$ respectively:
\[
P_{1}(x^n) = 
\begin{cases}
\frac{1}{2^n} - \frac{1}{2^{n+2}}, & \text{if $n$ is odd},\\[6pt]
\frac{1}{2^n} + \frac{1}{2^{n+1}}, & \text{if $n$ is even},
\end{cases}
\]
and
\[
P_{2}(x^n) = 
\begin{cases}
\frac{1}{2^n} - \frac{1}{2^{n+1}}, & \text{if $n$ is even},\\[6pt]
\frac{1}{2^n} + \frac{1}{2^{n+2}}, & \text{if $n$ is odd}.
\end{cases}
\]
One can verify that each of these defines a proper distribution over $n=1,2,3,\dots$ (the probabilities sum to 1 for each, since the alternating added and subtracted terms cancel out telescopically). Importantly, note that:
\begin{itemize}
    \item For both $L_{p,1}$ and $L_{p,2}$, every string $x^n$ with $n\ge1$ has non-zero probability. Thus, $\mathop{\mathrm{supp}}(L_{p,1}) = \mathop{\mathrm{supp}}(L_{p,2}) = L$ (the set of all non-empty strings of $x$).
    \item The distributions differ only in the \emph{sign} of the small adjustments. For instance, for $n$ even, $P_1(x^n)$ is a bit larger than the baseline $1/2^n$, whereas $P_2(x^n)$ is a bit smaller; for $n$ odd, it’s the opposite.
\end{itemize}
Now, suppose we have an identification algorithm $\mathcal{A}$ that sees an infinite presentation of strings that come from one of these two probabilistic languages (but $\mathcal{A}$ does not know which one). By the definition of identification in the limit, the presentation is an adversarial sequence that contains every string in $L$ at least once (since the support is $L$ itself). Crucially, because $\mathcal{A}$ must succeed for \emph{any} sequence that meets this criterion, we can adversarially choose the order in which strings appear. In particular, we can arrange for $\mathcal{A}$ to see all possible strings in some order that completely hides any statistical bias. For example, the adversary could present the strings in lexicographic order (or by increasing length). In such an arrangement, the data sequence $\langle s_1, s_2, \dots \rangle$ would be \emph{the same} whether the underlying source is $L_{p,1}$ or $L_{p,2}$, since both can generate all those strings.

Since $L_{p,1}$ and $L_{p,2}$ have identical supports, any sequence of distinct strings from that support is consistent with either model having produced it. The identification algorithm $\mathcal{A}$, by requirement (2) of Definition~\ref{def:gold_prob}, cannot eliminate either hypothesis until it encounters some evidence that contradicts one of them. But there will be no such evidence in terms of observed strings, because any finite set of observed strings $S$ will be a subset of $L$, and both hypotheses $L_{p,1}$ and $L_{p,2}$ are consistent with $S$.

Therefore, $\mathcal{A}$ can never be sure whether the source is $L_{p,1}$ or $L_{p,2}$. We can even say that for any strategy $\mathcal{A}$ might use to eventually choose one of the two, we (as an adversary) can define the data sequence in such a way that $\mathcal{A}$’s final guess is wrong. For instance, if $\mathcal{A}$ plans to guess “Model~1” after seeing some sufficiently long prefix of data, we can orchestrate the data (which, again, does not violate consistency with Model~2) such that $\mathcal{A}$ is misled. In other words, whichever model $\mathcal{A}$ converges on, we ensure the actual model was the other one.

Thus, no identification algorithm can guarantee to output the correct model $L_{p,1}$ vs $L_{p,2}$ after finitely many steps on all valid input sequences. This proves that the class $\{L_{p,1}, L_{p,2}\}$ is not identifiable in the limit.
\end{proof}

The above proof shows a pathological but illuminating construction: we made two models that are indistinguishable by language membership alone yet differ in their probabilistic structure. The adversary exploited the fact that identification in the limit does not assume random sampling according to the model distribution; if instead we did assume the samples were drawn from the model’s distribution, then over time one might detect the slight differences in frequencies (this would move us into the territory of probabilistic identification or Bayesian inference with enough data). But under Gold’s paradigm (which is adversarially robust, considering worst-case presentation of data), the difference in distributions is irrelevant because the adversary can always present data in a way that conceals it. Notably, this is the first example of identification failure in a finite hypothesis class of probabilistic languages. Previously, impossibility results (e.g. Gold 1967) required an infinite or uncountable class, or deterministic outputs – our counterexample shows even two stochastic models can confound any attribution attempt. This closes a long-standing gap in the theory.

From an LLM attribution standpoint, this example captures a real concern: modern generative models (especially fine-tuned ones) often have very similar capabilities, differing mostly in the probabilities they assign to various outputs or in subtleties of style. If two models $M_1$ and $M_2$ have been trained on largely overlapping data, $M_1$ might respond to a prompt almost the same way as $M_2$, with only slight differences in phrasing probabilities. Any fixed set of outputs that $M_1$ can produce, $M_2$ might also be able to produce (especially if the prompt is chosen adversarially to maximize confusion). Our theoretical result shows that in the worst case, if their output supports are identical, an attacker who only sees whether an output happened or not (not how often among many trials) cannot distinguish them. 

It is worth noting that the proof’s construction might seem artificial (the probability mass functions were carefully designed). But one can imagine more natural scenarios: for example, two language models that have the \emph{exact same range of expression} (say, both have memorized the same set of internet texts), but one is fine-tuned to prefer certain styles more than the other. Without many samples to do statistical analysis, any single piece of text produced by one model could also have been produced by the other. Identification in the limit says we can have as many samples as we want in the long run, but since the adversary controls the ordering, they can always choose a diverse set that doesn’t reveal the biases. In effect, the adversary can simulate the distribution of the other model by interleaving outputs.

The proof above relies on the two probabilistic languages having identical support (as formal languages). However, this situation captures well the situation with LLMs where one can force them to output almost anything. In other words, if the models are sufficiently expressive (e.g., large GPT-style models can be prompted to talk about almost any topic), then differences lie not in what they \emph{can} say but in what they \emph{tend} to say (sometimes refered as the \emph{propensity} of an LLM). If one model can be coerced (through prompt or context) to imitate the style of another, then purely from the fact that a certain output was observed, we gain no information about which model was behind it.

In summary, we have established that:
\begin{itemize}
    \item If we have infinitely many possible models, attribution is theoretically impossible (Sections~2 and~3).
    \item If we have finitely many models \emph{and} treat them as producing deterministic languages, attribution is possible in principle (Section~4).
    \item If we have finitely many models but they produce outputs probabilistically and have overlapping capabilities, attribution can again become impossible (this section).
\end{itemize}
This paints a rather bleak theoretical picture: the only regime that avoids impossibility is the one with a finite, discrete set of hypotheses. In practice, the space of potential models (especially fine-tuned variants) is enormous and effectively unbounded, and models are stochastic. Therefore, the negative results seem most relevant to real-world conditions. Next, we turn to empirical evidence to quantify just how large the hypothesis space has grown, and we analyze the computational limits of brute-force attribution approaches.

\section{The Rapid Expansion of the LLM Hypothesis Space}
Thus far, our theoretical analysis suggests that unless the set of candidate models is very limited, one cannot reliably attribute a given output to its true source in the worst case. In practice, the ecosystem of LLMs is anything but limited: it is rapidly growing, with new models and fine-tuned versions emerging constantly. In this section, we present a data-driven analysis that illustrates the scale of the problem.

We assembled data from Stanford University's \emph{Ecosystem Graphs} project~\cite{bommasani2023ecosystem-graphs}, which documents released AI models and datasets over time. The dataset we use includes information on 359 language models (as of January 2025) and 112 datasets, among other assets. For our purposes, we focus on:
\begin{itemize}
    \item The number of distinct base models (checkpoints) released, including whether they are open-source or closed-source.
    \item The number of distinct datasets released over time.
\end{itemize}
Using these, we can estimate how many potential fine-tuned variants could exist by combining open models with available datasets.

Specifically, let:
\[ C(t) = \#\{\text{closed or restricted models released up to time }t\}, \]
\[ O(t) = \#\{\text{open-source model checkpoints released up to time }t\}, \]
\[ D(t) = \#\{\text{datasets released up to time }t\}. \]
Here, $C(t)$ and $O(t)$ partition the total number of models by accessibility (closed vs open), and $D(t)$ measures the cumulative count of distinct datasets.

A \emph{conservative} scenario for the growth of distinct fine-tuned models is to assume each open-source model can be fine-tuned on at most one dataset. In that case, by time $t$ the total number of (base model or fine-tuned) model variants is at least:
\[
N_{\mathrm{single}}(t) = C(t) \;+\; O(t)\,\big[\,1 + D(t)\big],
\] 
where $O(t)$ models can each give rise to at most one fine-tuned variant on each of the $D(t)$ datasets, hence $O(t)\,D(t)$ possible fine-tunes, plus the base models themselves ($O(t)$) and the closed models $C(t)$ which might not be fine-tuned openly. We add $1$ in the bracket to count the base open models themselves (fine-tuned on “no new data”).

We can also consider more aggressive combinations. If each open model could be fine-tuned on up to two datasets (combined), the number of possible distinct outputs grows combinatorially:
\[
N_{k\le2}(t) = C(t) \;+\; O(t)\,\Big[\,1 + D(t) + \binom{D(t)}{2}\Big],
\] 
where $\binom{D(t)}{2}$ is the number of ways to pick two distinct datasets to jointly fine-tune. Similarly, allowing up to three datasets per fine-tune:
\[
N_{k\le3}(t) = C(t) \;+\; O(t)\,\Big[\,1 + D(t) + \binom{D(t)}{2} + \binom{D(t)}{3}\Big].
\]

These formulas provide lower bounds on the number of \emph{conceptually distinct} model variants one could have by mixing available ingredients (models and datasets). Not every combination is actually realized or made public, of course, but they represent the search space size that an attribution mechanism might have to contend with if an adversary could fine-tune or modify models arbitrarily using existing data.

Using our collected data, we computed these metrics from 2018 through the beginning of 2025. The results are striking:

\subsection{Global Growth in Candidate Models}
\begin{figure}[htbp]
    \centering
    \includegraphics[width=0.7\linewidth]{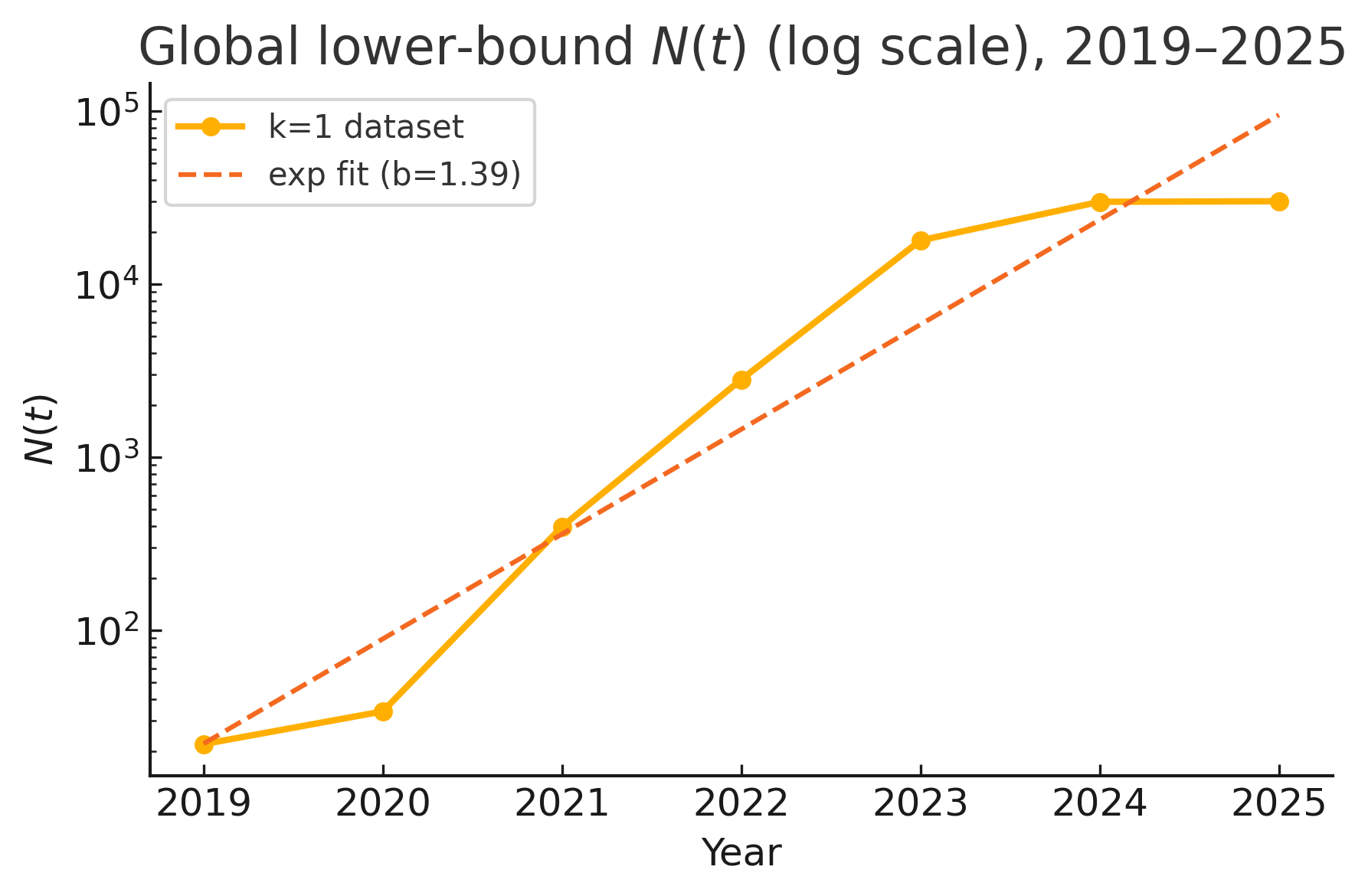}
     \caption{Growth of the combinatorial lower bound \(N_{\mathrm{single}}(t)\) (single-dataset fine-tunes per open model) from 2019 to 2025 on a logarithmic scale. The dashed line shows a least-squares exponential fit with estimated growth rate \(b = 1.39\,\mathrm{yr}^{-1}\), corresponding to a doubling time of \(\tau = \SI{0.50}{\year}\).}

    \label{fig:global}
\end{figure}

Figure~\ref{fig:global} shows the exponential growth of
\(N_{\mathrm{single}}(t)\), the conservative count of model variants
under a single-dataset fine-tune assumption.  Even so, the hypothesis
space swells from on the order of \(10^1\) plausible variants in 2019 to
\(10^4\) by 2023, and surpasses \(3\times10^4\) by 2025.  A least-squares
exponential fit (\(R^2 \approx 0.97\)) yields a growth rate
\(b = 1.39\:\mathrm{yr}^{-1}\), corresponding to a doubling time
\(\tau = \ln2/b \approx 0.50\:\mathrm{yr}\) (roughly six months).  In
other words, the effective hypothesis space for attribution doubles
twice per calendar year, far outpacing the pace of any brute-force
inspection.  

\subsection{Fine-Tuning on Multiple Datasets}
\begin{figure}[htbp]
    \centering
    \includegraphics[width=0.7\linewidth]{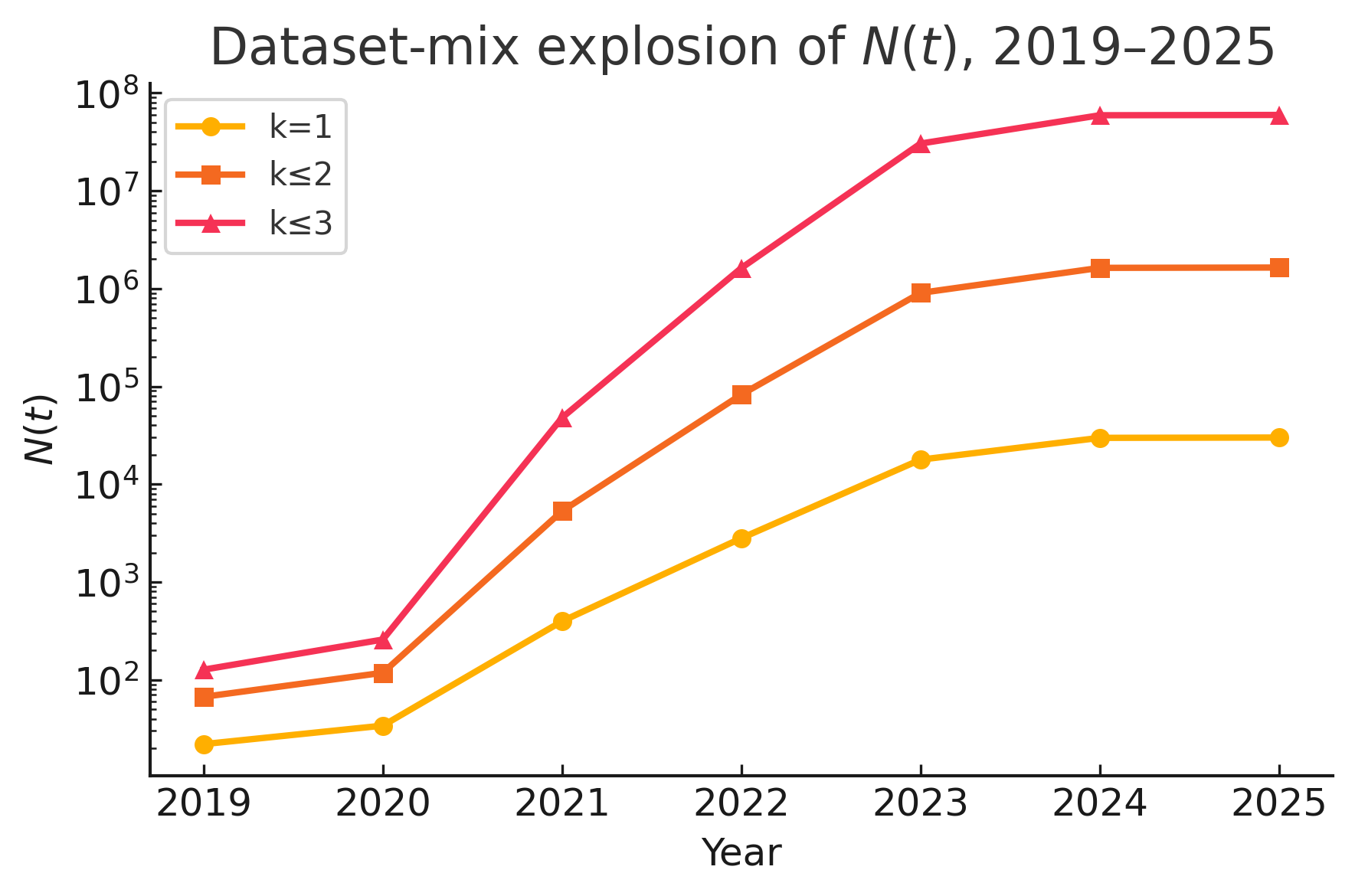}
    \caption{Projected growth of the conservative lower bound \(N(t)\)
           when permitting fine-tuning on combinations of up to
           \(k=2\) (squares) or \(k=3\) (triangles) datasets per
           checkpoint, from 2019 to 2025 on a logarithmic scale.  Under
           these assumptions, \(N_{k\le2}\) reaches
           \(\sim10^{6}\) variants by 2025 (doubling time
           \(\tau\approx0.36\) yr), and \(N_{k\le3}\) approaches
           \(\sim10^{7}\) (doubling time \(\tau\approx0.28\) yr).}
    \label{fig:mix}
\end{figure}

Figure~\ref{fig:mix} depicts the explosive growth of our combinatorial
lower bound \(N(t)\) when open-weight checkpoints are allowed to be
fine-tuned on up to two or three datasets.  In the \(k\le2\) scenario,
\(N_{k\le2}(t)\) reaches on the order of \(10^{6}\) variants by 2025,
while for \(k\le3\), \(N_{k\le3}(t)\) approaches \(10^{7}\).  Exponential
fits on \(\ln N\) (both with \(R^2\approx0.97\)) give growth rates
\[
  b_{k\le2} \approx 1.95~\mathrm{yr}^{-1}, \quad
  b_{k\le3} \approx 2.51~\mathrm{yr}^{-1},
\]
which correspond to doubling times
\(\tau_{k\le2} = \ln2 / b_{k\le2} \approx 0.36\:\mathrm{yr}\)
and
\(\tau_{k\le3} = \ln2 / b_{k\le3} \approx 0.28\:\mathrm{yr}\).
Thus, even a minimal two-dataset allowance doubles the hypothesis space
in under four months, and three-way mixes halve that interval to just
over three months.  Any brute-force attribution approach would soon be
outstripped by this relentless combinatorial explosion.  

In practical terms, even if only a small fraction of these combinations actually exist as deployed models, an attacker could potentially fine-tune a new model on a novel mix of data to evade known detectors. The defender, lacking knowledge of that specific fine-tune, would have to consider it as a possibility among an astronomically large set.

\subsection{Trends by Modality and Region}

To pinpoint which segments drive the hypothesis-space explosion, we
classified each checkpoint by its primary \emph{modality} (text, vision,
multimodal, audio, other, unknown) and by the developer’s geographic
\emph{region} (North America, Europe, Asia, Other).  We then recomputed
the conservative lower bound \(N_{\mathrm{single}}(t)\) for each slice.

\begin{figure}[htbp]
  \centering
  \begin{subfigure}[b]{0.48\textwidth}
    \includegraphics[width=\linewidth]{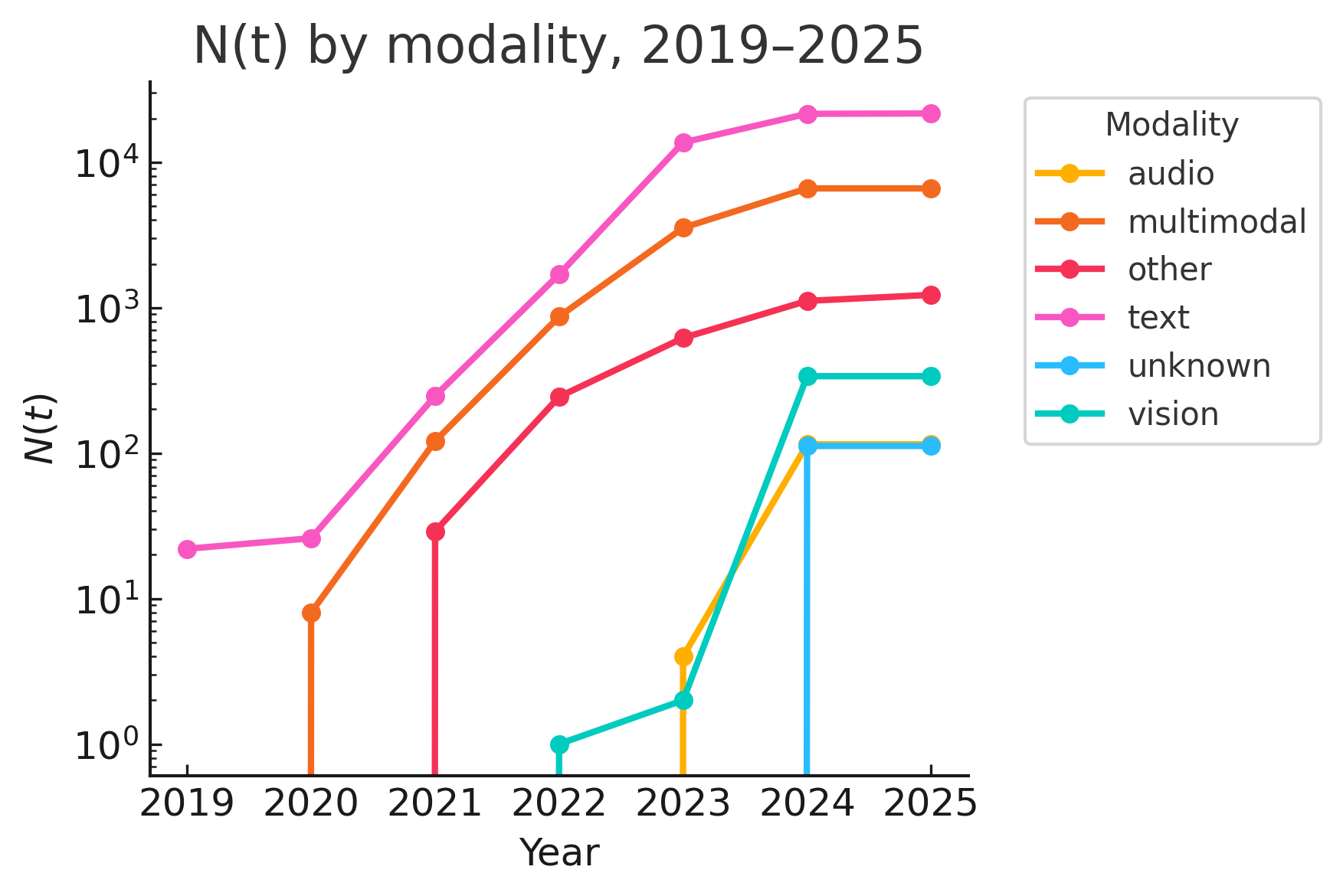}
    \caption{}
    \label{fig:modality}
  \end{subfigure}\hfill
  \begin{subfigure}[b]{0.48\textwidth}
    \includegraphics[width=\linewidth]{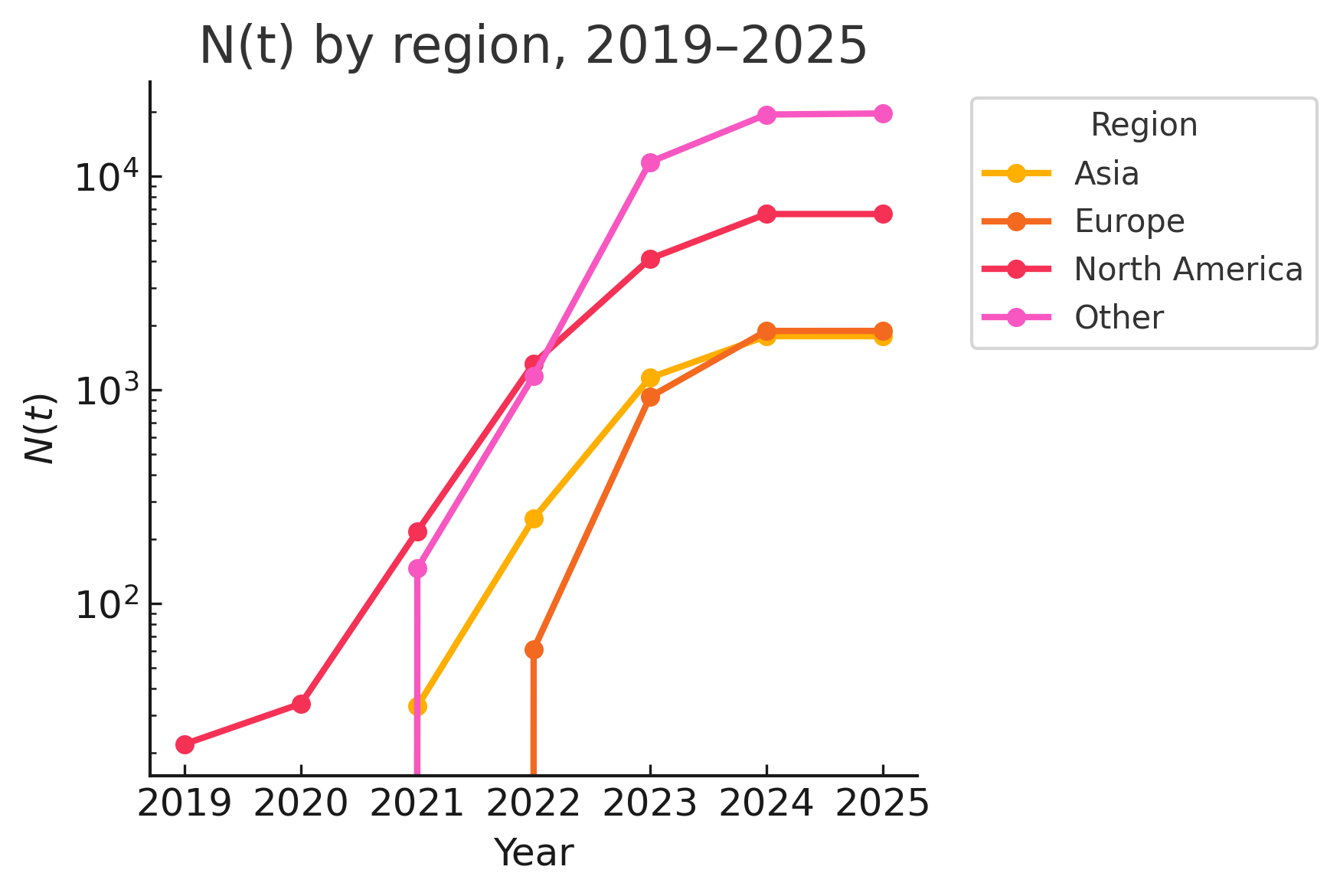}
    \caption{}
    \label{fig:region}
  \end{subfigure}
  \caption{Sub-population breakdowns of the attribution hypothesis
           space.  (a) Growth of \(N_{\mathrm{single}}(t)\) by model modality,
             2019–2025, on a log scale.  Text-only variants remain the
             largest class, but multimodal models exhibit the steepest
             relative growth, signaling imminent attribution
             challenges in image+text domains. (b) Growth of \(N_{\mathrm{single}}(t)\) by developer region,
             2019–2025, on a log scale.  North America led early
             releases, but Asia’s rapid uptake of open-weight models
             has narrowed the gap by 2024.  Europe and Other regions
             also show steady contributions.}
\end{figure}

Figure~\ref{fig:modality} demonstrates that while text models still
dominate in absolute count, the modal segment growing fastest (by slope)
is \emph{multimodal}, followed by vision and audio.  This trend warns
that non-text attribution (e.g.\ image and video) will soon face the
same combinatorial explosion as language models.

Figure~\ref{fig:region} shows a clear shift from a North-America–centric
model ecosystem toward a more balanced, global landscape.  Asian
organizations, particularly open-source contributors, now match or
exceed North American counts by late 2024.  Europe and Other regions
add further breadth.  Attribution frameworks must therefore be prepared
for suspect generators worldwide, not just from the major U.S. players.

These previously undocumented trends spotlight emerging areas (e.g. non-text modalities, non-Western model providers) where attribution efforts will face growing complexity.” 

\subsection{Exponential Growth Summary}
To summarize the quantitative growth, Table~\ref{tab:fits} reports the
exponential fit parameters for the overall $N(t)$ under different
fine-tuning assumptions from 2019 to 2025.  All scenarios show
$R^2 \approx 0.97$, indicating an almost perfect exponential trend.

\begin{table}[htbp]
\centering
\caption{Estimated exponential growth parameters for the hypothesis
         space $N(t)$ from 2019 to 2025 under different fine-tuning
         assumptions.  $b$ is the exponential growth rate per year,
         and $\tau = \ln2 / b$ is the doubling time.}
\label{tab:fits}
\begin{tabular}{lccc}
\toprule
\textbf{Metric} & \textbf{$b$ (yr$^{-1}$)} & \textbf{$R^2$} &
\textbf{Doubling time $\tau$ (yr)}\\
\midrule
$N_{\mathrm{single}}$ (1 dataset)   & 1.39 & 0.97 & 0.50 \\
$N_{k\le2}$ (up to 2 datasets)      & 1.95 & 0.97 & 0.36 \\
$N_{k\le3}$ (up to 3 datasets)      & 2.51 & 0.97 & 0.28 \\
\bottomrule
\end{tabular}
\end{table}

These figures underscore an unsustainable trajectory: even in the most
conservative single-dataset scenario, the hypothesis space doubles in
under six months.  With two-dataset mixes, it doubles in just over
four months, and with three-dataset mixes in under three and a half
months.  Any forensic catalog or fingerprinting system will struggle
to keep pace with such rapid doubling.

\section{Computational Feasibility of Exhaustive Attribution}

\begin{figure}[htbp]
  \centering
  \includegraphics[width=0.7\linewidth]{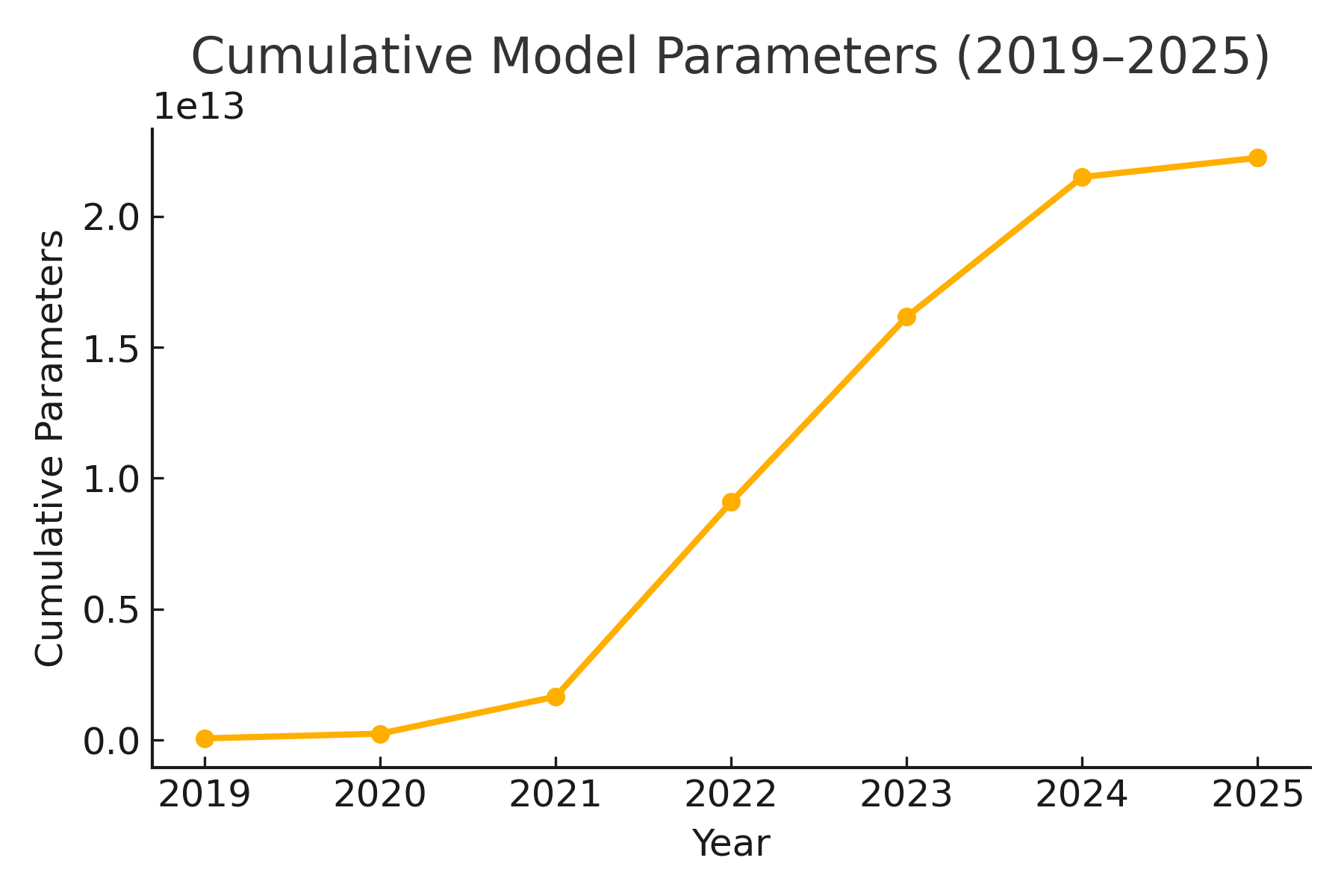}
  \caption{Cumulative total of model parameters (log scale) from 2019 to
           2025, with discrete year ticks.  The curve illustrates a
           two-order-of-magnitude jump in just three years,
           underscoring the steep “parameter cliff”.}
  \label{fig:param_cliff}
\end{figure}

To fully understand the challenge of attributing generated content to its original model, we must consider not just the combinatorial explosion of plausible model variants, but also the raw computational cost involved. A straightforward, brute-force approach to attribution—evaluating the likelihood of a given output against each known model—quickly becomes untenable as the number and size of models increase.

As a concrete baseline, we extracted parameter counts for all models available in our comprehensive ecosystem dataset through 2025. Early on, the cumulative size of all known models was modest: approximately $1.3 \times 10^{10}$ (13 billion) parameters in 2019. Yet by 2025, this figure surged dramatically to roughly $2.2 \times 10^{13}$ (22 trillion) parameters—an increase of nearly three orders of magnitude in just six years (Figure \ref{fig:param_cliff}). To estimate this growth, we conservatively imputed missing parameter sizes with the average of known model sizes, ensuring our estimate remains robust yet conservative.

\begin{figure}[htbp]
\centering
  % ---- (a) Heat-map (wider) ----
  \begin{subfigure}[b]{0.60\linewidth}
    \includegraphics[width=\linewidth]{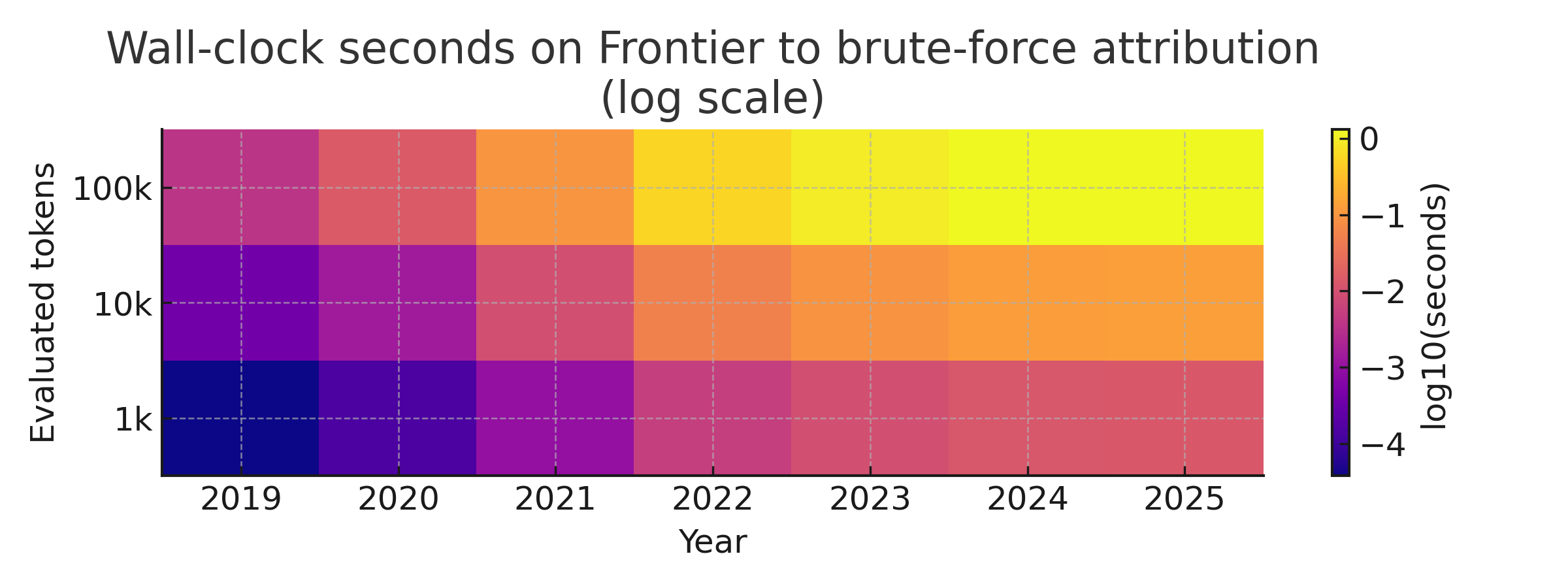}
    \label{fig:heatmap}
  \end{subfigure}\hfill
  % ---- (b) Hours curve (narrower) ----
  \begin{subfigure}[b]{0.38\linewidth}
    \includegraphics[width=\linewidth]{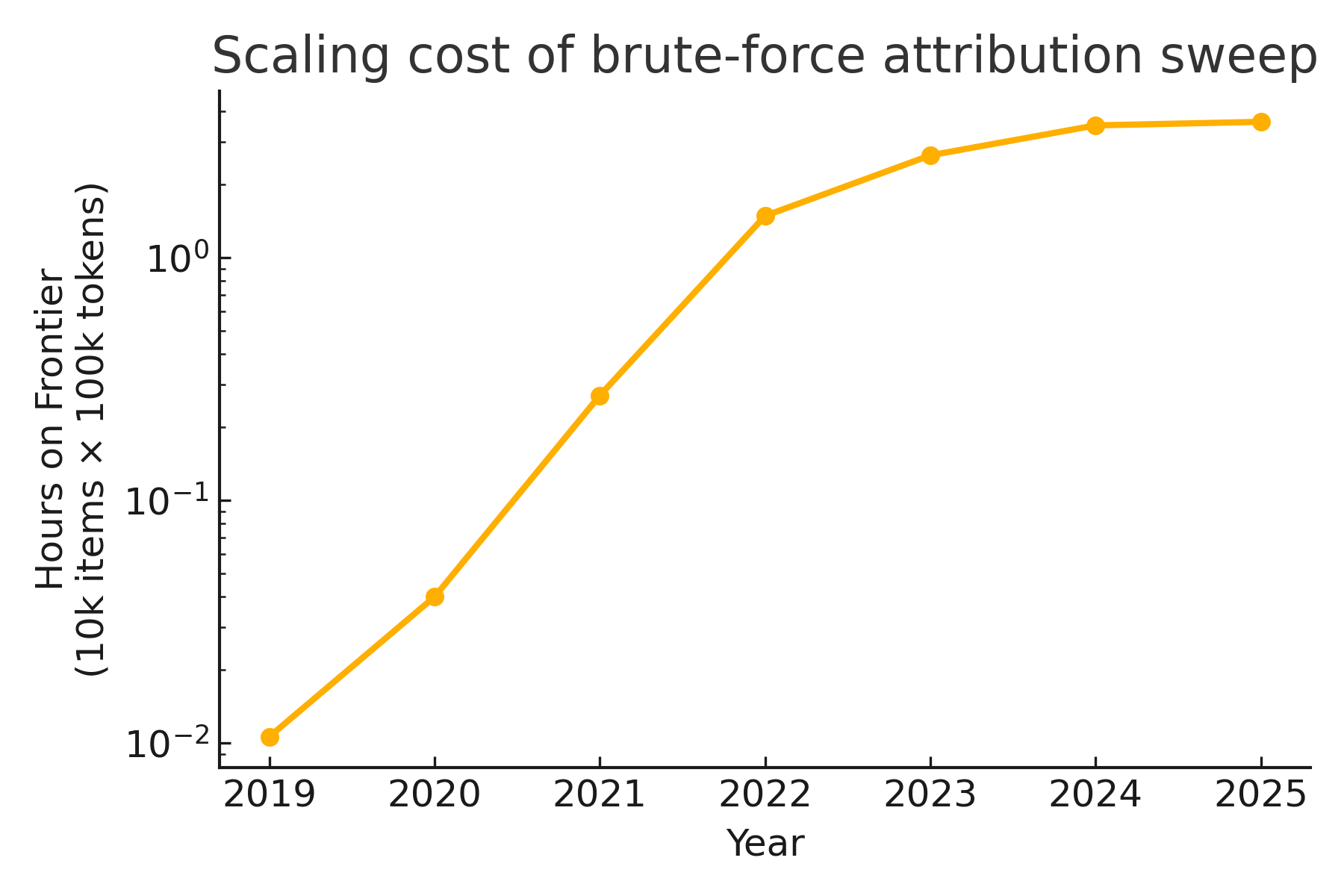}
    \label{fig:sweephours}
  \end{subfigure}

  \caption{Compute wall for exhaustive attribution. 
  (a)  Log$_{10}$ wall-clock seconds on Frontier required to brute-force
  the likelihood of a \emph{single} sequence against \emph{all} known models
  each year (rows: 1 k, 10 k, 100 k tokens).  
  (b)  Wall-clock \emph{hours} on Frontier needed to sweep
  10 000 suspect items of 100 k tokens each day as the model ecosystem
  grows from 2019 to 2025 (log scale).  Together, the panels show how
  single-item attribution creeps from milliseconds to seconds, while a
  modest daily workload balloons from minutes to multi-hour jobs—well
  beyond real-time forensic tolerances.}
  \label{figcompute}
\end{figure}

What would this mean in practical computational terms? Consider the task of attributing a single piece of suspicious content consisting of 100,000 tokens—a sizable but plausible length for a lengthy piece of misinformation or propaganda. Under a naive scenario where we calculate a single floating-point operation (FLOP) per parameter per token to compute likelihoods (which is highly optimistic—real-world inference is several times costlier), we would require $2.2 \times 10^{13}$ parameters multiplied by $10^5$ tokens, totaling about $2.2 \times 10^{18}$ FLOPs for this single attribution.

To contextualize this number, the Frontier supercomputer—among the most powerful systems available as of 2025 \cite{frontier_supercomputer}—achieves peak performance of approximately $1.7 \times 10^{18}$ FLOPs per second. Thus, attributing just one suspicious 100k-token output across all known models in 2025 would theoretically take around $2.2 \times 10^{18} / 1.7 \times 10^{18} \approx 1.3$ seconds on Frontier under optimal conditions (Figure \ref{figcompute}(a)). While this might appear reasonable for isolated incidents, the complexity scales steeply.
Consider a more demanding case in which \emph{10 000} suspicious
outputs—each 100 k tokens long—must be checked every day.  With the
2025 model set (~\(2.2\times10^{13}\) parameters), that workload
requires
\[
10^{4}\;\text{items}\times 10^{5}\;\text{tokens/item}\times
2.2\times10^{13}\;\text{params}
\;=\;
2.2\times10^{22}\ \text{FLOPs/day}.
\]
At Frontier’s peak \(1.7\) exaFLOP s\(^{-1}\) this equals
\(2.2\times10^{22}/1.7\times10^{18}\approx1.3\times10^{4}\,{\rm s}\),
or \(\mathbf{3.6}\) h of wall-time—about \(\mathbf{15\%}\) of the
machine’s entire daily capacity for a \emph{single} attribution task-set
(Fig.~\ref{figcompute}b).

Scaling to nationwide monitoring renders brute force hopeless.  Recent
surveys suggest \(1.32\times10^{8}\) U.S.\ adults use
generative-AI daily.  At a modest 10 000 tokens each, the country
produces
\[
1.32\times10^{8}\times10^{4}=1.32\times10^{12}\ \text{tokens/day},
\]
or \(4.8\times10^{14}\) tokens per year.  Brute-forcing those tokens
against the same 22 T-parameter pool would need
\[
4.8\times10^{14}\times2.2\times10^{13} \approx 1.1\times10^{28}\ \text{FLOPs}.
\]
Even with Frontier,
\[
\frac{1.1\times10^{28}}{1.7\times10^{18}}
\approx 6.5\times10^{9}\,\text{s}\approx\mathbf{200\;years}
\]
of uninterrupted peak compute would be required to attribute a
\emph{single} year’s U.S.\ output—clearly infeasible in real time.

Merely \emph{streaming} that much data%
\,(1.9\,\text{PB at }4\,\text{B/token})%
through Frontier’s \(\sim 2\,\text{TB}\,\text{s}^{-1}\) burst~I/O
would already consume about \(16\,\text{min}\) per day.

\begin{table}[htbp]
\centering
\renewcommand{\arraystretch}{1.15}
\begin{tabular}{p{0.55\textwidth} p{0.35\textwidth}}
\toprule
\textbf{Metric (U.S., annualised 2025)} & \textbf{Estimated Value} \\ \midrule
Daily Active LLM Users & $1.32 \times 10^{8}$ \\
Total Tokens Generated per Year & $4.8 \times 10^{14}$ tokens \\[0.3em]
\textit{--- Pure I/O scan ---} & \\
Data Volume (4 bytes / token) & $\approx 1.9$ PB \\ 
Streaming Time on Frontier\footnotesize{\,(\(\sim\!2\) TB s\(^{-1}\) burst)} & $\sim\!16$ min \\[0.6em]
\textit{--- Brute-force likelihood (22 T parameters) ---} & \\
Total FLOPs ($4.8\!\times\!10^{14}$ tokens $\times$ $2.2\!\times\!10^{13}$ params) & $1.1 \times 10^{28}$ FLOPs \\
Processing Time on Frontier (1.7 exaFLOP s\(^{-1}\)) & $\sim\!6.5 \times 10^{9}$ s $\;\approx\;$200 yr \\ \bottomrule
\end{tabular}
\caption{Compute budget for a \emph{single, exhaustive} annual sweep of
all LLM-generated content in the United States.  Even with the Frontier
supercomputer, simple I/O takes minutes per day, while full likelihood
evaluation would require centuries.}
\label{tab:attribution_time_estimation}
\end{table}

This computational impossibility of exhaustive attribution emphasizes a strategic dilemma: brute-force attribution, already at the brink of infeasibility today, will soon become outright impossible due to relentless increases in model quantity, diversity, and parameter size. 
Of course, real attribution might not require scanning everything; one might triage or focus on certain content. But the point remains that any brute-force approach---even one that assumes access to all model internals and unlimited computing---faces extreme scalability issues. In practice, things are even harder: many models are closed or not publicly runnable, content might be encrypted or privacy-protected, and one would have to find clever shortcuts.

\section{Broader Challenges and Discussion}
Beyond the theoretical impossibilities and the combinatorial and computational explosions described above, several other factors complicate LLM attribution:

\paragraph{Evasion via Social Network Dynamics.} 
In realistic attack scenarios, malicious content generated by LLMs will propagate through social networks and other channels before reaching its victims.  Practically, investigators may have only a single snippet of suspect text to examine, and our impossibility results show that if attribution can fail even with unlimited data in theory, one sample is clearly far from sufficient.  Attackers exploit this by limiting the evidence they reveal—distributing an operation across many small pieces of content or relaying it through social-network chains—so defenders never obtain a rich dataset to analyze.  Attackers can also inject AI-generated disinformation via proxy accounts or compromised nodes, making the origin hard to trace.  The structure of social networks—often characterized by small-world properties and heavy-tailed connectivity \cite{newman2006structure}—allows content to diffuse widely without clear indication of where it started.  Adversaries may deliberately rewire parts of the network or introduce \emph{Sybil} nodes (fake accounts) to obfuscate information flow \cite{waniek2022social,waniek2022trading}.  These strategies echo classical problems in network forensics: tracking the “patient zero” of an epidemic or the first source of a rumor is notoriously difficult without strong assumptions \cite{christakis2009connected}.  In the context of AI, even if perfect model attribution were possible in principle, content hopping through many users risks attributing the text to the last sharer rather than the true originator, defeating accountability.

\paragraph{A Generative Arms Race.} While attribution of content to models appears hard, the \emph{generation} of content is getting easier. Fascinatingly, recent theoretical work by Kleinberg and Mullainathan~\cite{kleinberg2024language} demonstrates what they call “language generation in the limit.” In a framework reminiscent of Gold’s, they show that it is possible for an agent to produce novel strings from a target language indefinitely without actually identifying the language. In plainer terms, one can imagine a situation where defenders deploy an AI to imitate an attacker’s model: the defender’s AI produces outputs that are valid under the attack model’s distribution, effectively matching the attacker’s capability, yet the defender’s AI has no idea which model it is imitating (it just knows how to continue the language). This theoretical result implies an odd stalemate: both the attacker and defender can spew out similar content, neither being able to conclusively prove which model is behind it. In such a scenario, attribution becomes a moot point---everyone can generate in the style of everyone else, and authenticity is lost in a sea of mimicry. This could lead to a game of confusion where any incriminating text could be dismissed as something an autonomous agent could have produced as well. It underscores how the very nature of generative AI erodes the link between model and output.

\paragraph{Toward Mitigation Strategies.} Given the grim outlook, what can be done? A few avenues, each with limitations, are often discussed:
\begin{itemize}
    \item \textbf{Watermarking and Fingerprinting.} If model developers voluntarily (or by regulation) incorporate hidden watermarks into generated content (e.g., detectable bit patterns in text or slight image perturbations), attribution could be greatly aided. Research on watermarking for LLMs is ongoing. However, watermarks can potentially be removed or spoofed by adversaries, especially if the watermark keys or methods become known. Moreover, not all model developers will comply, especially open-source ones or malicious actors.
    \item \textbf{Model Authentication Infrastructure.} One could imagine a world where each model has a cryptographic signature and each piece of AI content comes with a certified origin stamp (like a cryptographic watermark or metadata). This is conceptually similar to code signing. However, this requires adoption across industry and doesn’t solve the problem of someone using an uncompromised model to produce content and then stripping the signature.

    \item \textbf{Reducing the Hypothesis Space.} From a policy angle, one way to make attribution easier is to limit the number of models at large. If, for example, only a handful of foundation models were authorized for public use, the identification problem shrinks to distinguishing among those (the finite deterministic case would be more applicable). This could be implemented via regulation or industry consolidation. However, this conflicts with open innovation and may be impractical to enforce globally.
\end{itemize}

Our results strongly suggest that purely technical, post-hoc attribution will not be a panacea. The rapid scaling of model capabilities and availability means we should assume a world where attribution of malicious AI outputs is difficult or impossible. This, in turn, argues for \emph{proactive} measures: for example, embedding protective measures in models, limiting access to highly capable models, educating the public about AI-generated misinformation, and building resilience to fake content.

\section{Conclusion}

We have explored the landscape of LLM attribution through both the lens of learning-theoretic impossibility and the empirical realities of today’s model ecosystem.  Building on Gold’s and Angluin’s identification-in-the-limit results, we proved that even with unlimited, unlabeled data one cannot in general infer the true source model unless the candidate set is artificially small and mutually exclusive.  Our new theorem shows that two probabilistic language models whose output distributions overlap everywhere are, in principle, indistinguishable.  This negative result is not a finite-sample curiosity; it is an information-theoretic limit that survives in the infinite-data regime.

The empirical picture amplifies this constraint.  Fine-tuning, checkpoint remixing, and automated architecture search are causing the pool of plausible generators to double every few months, while brute-force attribution already exceeds exascale budgets for anything beyond toy corpora.  Exhaustive search is therefore a non-starter, and statistical “best-guess’’ methods inherit the impossibility boundary just established.

These findings collide with a policy environment that increasingly assumes model-level traceability.  The EU AI Act, U.S.\ Executive Order 14110, and parallel draft rules in the U.K.\ and OECD impose obligations—incident reporting, watermark retention, and red-team documentation—squarely at the provider or model level \cite{UKAIWhitePaper2023,OECDAIRecommendation2024,OECDHAIPFramework2025}.  Without reliable attribution, regulators lack the technical substrate to levy fines or mandate recalls, and developers cannot demonstrate the “due diligence’’ now embedded in many safe-harbor clauses.  Chain-of-custody investigations face a similar impasse: modern attack pipelines often route prompts through multiple agents, so isolating the compromised node may be the only practical way to halt a live exploit when the human operator is anonymous.  For instance, in March 2024 security researchers at Cornell Tech disclosed “Morris II,” a self-replicating AI worm that jumped from one generative-AI email assistant to the next by embedding a malicious prompt in each message. Because every assistant (GPT-4, Gemini Pro, LLaVA) automatically forwarded the adversarial prompt downstream, investigators found the only viable containment strategy was to quarantine the first compromised agent—the human operator never had to re-enter the loop.,\cite{Cohen2024MorrisII,Burgess2024Wired}. If the identification of this first compromised agent is hard, that provides further difficulty in these sorts of investigations. 

From the developer’s perspective, proving that a particular model was—or was not—responsible for some output is increasingly central to liability and reputation.  In January 2024, for instance, an AI-generated robocall cloned President Biden’s voice and urged New Hampshire voters to skip the primary; independent analysts at Pindrop and UC Berkeley traced spectral artefacts to the commercial voice-cloning startup ElevenLabs, but the company itself said it “cannot comment on specific incidents,” hinting at its inability to decisively confirm or deny authorship in real time \cite{WiredBidenRobocall2024,PindropRobocall2024,WPBidenProbe2024}  Our impossibility results therefore imply that beyond a certain point, developers cannot rely on post-hoc attribution to catch misuses—strengthening the case for preventative measures such as stricter release controls or built-in watermarking for powerful models.

Because the impossibility bound is information-theoretic, it suggests a pivot from post-hoc identification to ex-ante safeguards.  Technical mitigations such as cryptographic signatures, robust watermarking, or permissioned key escrow inject additional bits that collapse the indistinguishability class.  Auditable usage logs and risk-tiered release practices further narrow the hypothesis space before an incident occurs.  The path forward, in other words, lies in engineering traceability rather than trying to infer it after the fact. Absent such measures, the true author of a text, whether human or AI (and if AI, which one), will frequently remain a mystery, complicating trust, enforcement, and democratic oversight in our information ecosystem. Just as early theoretical studies of cybersecurity—on conflict timing and the strategic dilemmas of digital attribution—proved prescient for later real-world incidents \cite{axelrod2014timing,edwards2017strategic}, we aim to anticipate LLM-attribution challenges before they become unmanageable.  A proactive grasp of these limits can guide the design of safer AI systems today.

\section*{Code availability.} 
All analyses and figures in this work can be fully with a
single Python notebook provided in 
(\href{https://colab.research.google.com/drive/1usA2Z72DnJfGAHb0zg4IlsrtfIS0oTX1?usp=sharing}{\texttt{Google~Colab}}),
which loads the \emph{Ecosystem Graphs} CSV and re-generates every plot
and FLOP calculation end-to-end.

\section*{Acknowledgements}

M.C. is funded by Horizon Europe Chips JU (HORIZON-JU-Chips-2024-2-RIA, NexTArc CAR), by grant PID2023-150271NB-C21 funded by MICIU/AEI/ 10.13039/501100011033 (Spanish Ministry of Science, Innovation and University, Spanish State Research Agency). 
ChatGPT assisted with editing, formatting, and code; all ideas and results are the authors’.


\begin{thebibliography}{99}\setlength{\itemsep}{0ex}
\bibitem{axelrod2014timing} Axelrod~R, Iliev~R (2014) Timing of cyber conflict. \textit{Proc Natl Acad Sci USA} \textbf{111}(4):1298--1303.
\bibitem{edwards2017strategic} Edwards~B, Furnas~T, Forrest~S, Axelrod~R (2017) Strategic aspects of cyberattack, attribution, and blame. \textit{Proc Natl Acad Sci USA} \textbf{114}(11):2825--2830.
\bibitem{urbina2022dual} Urbina~F, Lentzos~F, Invernizzi~C, Ekins~S (2022) Dual use of AI-powered drug discovery. \textit{Nat Mach Intell} \textbf{4}(3):189--191.
\bibitem{perlroth2021they} Perlroth~N (2021) \textit{This Is How They Tell Me the World Ends: The Cyberweapons Arms Race} (Bloomsbury, New York).
\bibitem{xu2024autoattacker} Xu~J, Stokes~JW, McDonald~G, Bai~X, Marshall~D, Wang~S, Swaminathan~A, Li~Z (2024) \textit{AutoAttacker: A Large Language Model Guided System to Implement Automatic Cyber-attacks} (arXiv:2403.01038).
\bibitem{rahwan2019machine} Rahwan~I, Cebrian~M \emph{et~al.} (2019) Machine behaviour. \textit{Nature} \textbf{568}(7753):477--486.
\bibitem{shavit2023practices} Shavit~L, Siddarth~D, Trager~R, Wolf~K (2023) Practices for governing agentic AI systems. (OpenAI, \texttt{https://cdn.openai.com/papers/Governing\_Agentic\_AI.pdf}).
\bibitem{anwar2024foundational} Anwar~U, Saparov~A, Rando~J, Paleka~D \emph{et~al.} (2024) Foundational challenges in assuring alignment and safety of large language models. (arXiv:2307.13744).
\bibitem{gold1967language} Gold~EM (1967) Language identification in the limit. \textit{Inf Control} \textbf{10}(5):447--474.
\bibitem{angluin1980inductive} Angluin~D (1980) Inductive inference of formal languages from positive data. \textit{Inf Control} \textbf{45}(2):117--135.
\bibitem{johnson2004gold} Johnson~K (2004) Gold's theorem and cognitive science. \textit{Philos Sci} \textbf{71}(4):571--592.
\bibitem{strobl2024formal} Strobl~L, Merrill~W, Weiss~G, Chiang~D, Angluin~D (2024) What formal languages can transformers express? A survey. \textit{Trans Assoc Comput Linguist} \textbf{12}:543--561.
\bibitem{bhattamishra2020computational} Bhattamishra~S, Ahuja~K, Goyal~N (2020) On the ability and limitations of transformers to recognize formal languages. In \textit{Proc EMNLP 2020}, pp. 7096--7116.
\bibitem{merrill2020effects} Merrill~W (2020) Linear transformations and the capacity of recurrent neural networks. In \textit{Proc ACL 2020 Workshop on Deep Learning and Formal Languages}.
\bibitem{peng2023limitations} Peng~B, Narayanan~S, Papadimitriou~C (2024) On the limitations of the transformer architecture. In \textit{Proc. 1st Conf. on AI for Language Modeling}, (to appear).
\bibitem{frontier_supercomputer} ORNL (2022) Frontier supercomputer debuts as world’s fastest, breaking exascale barrier. (Oak Ridge National Lab news release, May 30, 2022).
\bibitem{bick2024rapid} Bick~A, Blandin~A, Mallen~J (2024) The rapid adoption of generative AI. (NBER Working Paper No. 32966).
\bibitem{newman2006structure} Newman~MEJ, Barabási~AL, Watts~DJ (2006) \textit{The Structure and Dynamics of Networks} (Princeton Univ. Press).
\bibitem{christakis2009connected} Christakis~NA, Fowler~JH (2009) \textit{Connected: The Surprising Power of Our Social Networks and How They Shape Our Lives} (Little, Brown).
\bibitem{waniek2022social} Waniek~M \emph{et~al.} (2022) Hiding individuals and communities in a social network. \textit{IEEE Trans Netw Sci Eng} \textbf{9}(1):196--209.
\bibitem{waniek2022trading} Waniek~M \emph{et~al.} (2022) Trading contact tracing precision for privacy via structure-preserving anonymization. \textit{Sci Rep} \textbf{12}:4.
\bibitem{cebrian2021origins} Cebrian~M \emph{et~al.} (2021) A time-critical crowdsourced computational search for the origins of COVID-19. \textit{Nat Electron} \textbf{4}(1):2--4.
\bibitem{kleinberg2024language} Kleinberg~J, Mullainathan~S (2024) Language generation in the limit. (arXiv:2404.06757).
\bibitem{bommasani2023ecosystem-graphs} Bommasani~R, Soylu~D, Liao~TI, Creel~KA, Liang~P (2024) Ecosystem Graphs: Documenting the foundation model supply chain. In \textit{Proc. 7th AAAI/ACM Conf. AI, Ethics, and Society (AIES 2024)}, pp. 196--207.
\bibitem{EUAIAct2024}
Regulation (EU) 2024/1689 of the European Parliament and of the Council of 13 June 2024 laying down harmonised rules on artificial intelligence and amending certain legislative acts (Artificial Intelligence Act), \textit{Official Journal of the European Union}, OJ L — 12 July 2024. Available at: \url{https://eur-lex.europa.eu/eli/reg/2024/1689/oj}.
\bibitem{EO14110}
Exec.\ Order No.\ 14\,110, \emph{Safe, Secure, and Trustworthy Development and Use of Artificial Intelligence}, 88~Fed.\ Reg.\ 75191\textendash75212 (Oct.\ 30, 2023). Available at: \url{https://www.federalregister.gov/documents/2023/11/02/2023-24766/safe-secure-and-trustworthy-development-and-use-of-artificial-intelligence}.
\bibitem{UKAIWhitePaper2023}
UK Department for Science, Innovation \& Technology.
\textit{AI Regulation: A Pro-Innovation Approach}. White Paper CP 815, 
29 March 2023 (updated 3 August 2023). 
Available at \url{https://www.gov.uk/government/publications/ai-regulation-a-pro-innovation-approach/white-paper}.
\bibitem{OECDAIRecommendation2024}
Organisation for Economic Co-operation and Development (OECD).
\textit{Recommendation of the Council on Artificial Intelligence (OECD/LEGAL/0449)}. 
Adopted 21 May 2019; revised 8 November 2023 and 3 May 2024. 
Available at \url{https://legalinstruments.oecd.org/en/instruments/oecd-legal-0449}. 
\bibitem{OECDHAIPFramework2025}
Organisation for Economic Co-operation and Development (OECD).
\textit{Global Framework to Monitor Application of the G7 Hiroshima AI Code of Conduct}. 
\bibitem{Cohen2024MorrisII}
Stav Cohen, Ron Bitton, and Ben Nassi.
\textit{Here Comes the AI Worm: Unleashing Zero-click Worms that Target GenAI-Powered Applications}.
arXiv:2403.02817, Cornell Tech, Mar. 2024.
Available at \url{https://arxiv.org/abs/2403.02817}. 
\bibitem{Burgess2024Wired}
Matt Burgess.
“Here Come the AI Worms.” 
\textit{WIRED}, 1 Mar. 2024. 
Available at \url{https://www.wired.com/story/here-come-the-ai-worms/}.
\bibitem{WiredBidenRobocall2024}
Kate Knibbs. 
“Researchers Say the Deepfake Biden Robocall Was Likely Made With Tools From AI Startup ElevenLabs.” 
\textit{WIRED}, 26 Jan 2024.  
Available at \url{https://www.wired.com/story/biden-robocall-deepfake-elevenlabs/}.  
\bibitem{PindropRobocall2024}
Vijay Balasubramaniyan. 
“Pindrop Reveals TTS Engine Behind Biden AI Robocall.” 
Pindrop Research Blog, 25 Jan 2024 (updated 16 July 2025).  
Available at \url{https://www.pindrop.com/article/pindrop-reveals-tts-engine-behind-biden-ai-robocall/}.  

\bibitem{WPBidenProbe2024}
Cat Zakrzewski and Pranshu Verma. 
“New Hampshire Opens Criminal Probe Into AI Calls Impersonating Biden.” 
\textit{The Washington Post}, 6 Feb 2024.  
Available at \url{https://www.washingtonpost.com/technology/2024/02/06/nh-robocalls-ai-biden/}.  
\end{thebibliography}
\end{document}